\documentclass[11pt]{article}
\usepackage[a4paper, total={6in, 8in}]{geometry}
\usepackage[utf8]{inputenc} % allow utf-8 input
\usepackage[T1]{fontenc}    % use 8-bit T1 fonts
\usepackage{hyperref}       % hyperlinks
\usepackage{url}            % simple URL typesetting
\usepackage{booktabs}       % professional-quality tables
\usepackage{amsfonts}       % blackboard math symbols
\usepackage{nicefrac}       % compact symbols for 1/2, etc.
\usepackage{microtype}      % microtypography
\usepackage{xcolor}         % colors

\usepackage{graphicx}
\usepackage{amsthm}
\usepackage{amsmath,amssymb}
\usepackage{natbib}
\usepackage{authblk}

\usepackage{color, colortbl}
\definecolor{Gray}{gray}{0.8}
\definecolor{LightCyan}{rgb}{0.88,1,1}
\definecolor{green(html/cssgreen)}{rgb}{0.0, 0.5, 0.0}

\newtheorem{remark}{Remark}
\newtheorem{theorem}{Theorem}

\theoremstyle{definition}
\newtheorem{definition}{Definition}[section]

\newcommand{\x}{\mathbf{x}}
\newcommand{\y}{\mathbf{y}}

\newcommand{\btheta}{\boldsymbol{\theta}}
\newcommand{\f}{\mathbf{f}}

\newcommand{\ddt}{\frac{\rm d}{{\rm d}t}}
\newcommand{\thetaDot}{\,\dot{\!\boldsymbol{\theta}}} % offsetting negative and positive spacing to center the dot. I know, I'm slightly OCD :) 

\newcommand{\distBound}{D}
\title{Generalization in Supervised Learning \\ Through Riemannian Contraction}

\author[1]{Leo Kozachkov \thanks{leokoz8@mit.edu}}
\author[2]{Patrick M. Wensing \thanks{pwensing@nd.edu}}
\author[1,3,4]{Jean-Jacques Slotine\thanks{jjs@mit.edu}}
\affil[1]{Department of Brain and Cognitive Sciences, MIT}
\affil[2]{Department of Aerospace and Mechanical Engineering, University of Notre Dame}
\affil[3]{Department of Mechanical Engineering, MIT}
\affil[4]{Google AI}
\date{}

\begin{document}

\maketitle

\begin{abstract}
We prove that Riemannian contraction in a supervised learning setting implies generalization. Specifically, we show that if an optimizer is contracting in some Riemannian metric with rate $\lambda > 0$, it is uniformly algorithmically stable with rate $\mathcal{O}(1/\lambda n)$, where $n$ is the number of labelled examples in the training set. The results hold for stochastic and deterministic optimization, in both continuous and discrete-time, for convex and non-convex loss surfaces. The associated generalization bounds reduce to well-known results in the particular case of gradient descent over convex or strongly convex loss surfaces. They can be shown to be optimal in  certain linear settings, such as kernel ridge regression under gradient flow.
\end{abstract}

%\section{MWE}

% Un-comment / Comment this next line
%\newcommand{\myconditional}{1}

%\ifdefined\myconditional
%Conditional Defined
%\else
%Conditional Not Defined
%\fi

\section{Introduction}
Our understanding of generalization in modern machine learning systems is lagging behind their empirical successes \citep{zhang2021understanding}.  These systems tend to be massively overparameterized, sometimes by several orders of magnitude \citep{allen2019learning,fedus2021switch}. It is therefore unsuprising that they can achieve zero training loss. What \textit{is} surprising is how well they can generalize. Their performance on held-out data is often very good. To understand this phenomenon, there has been an influx of theoretical research into establishing generalization bounds for iterative optimization algorithms such as gradient descent. In this work we focus on optimizers within a supervised learning setting, where we are given access to a number of labelled training points drawn from some underlying common distribution, as well as a loss function which quantifies performance. Within such a setting, we show that if an optimizer is \textit{contracting} \citep{lohmiller1998contraction} in some Riemannian metric (in a precise sense defined below) then it is algorithmically stable in every metric. Our theory applies to wide variety of common optimizers--for example gradient flows and stochastic minibatch gradient descent--operating over both convex and non-convex loss surfaces.

\begin{figure}[ht]
\centering
\includegraphics[width = \textwidth]{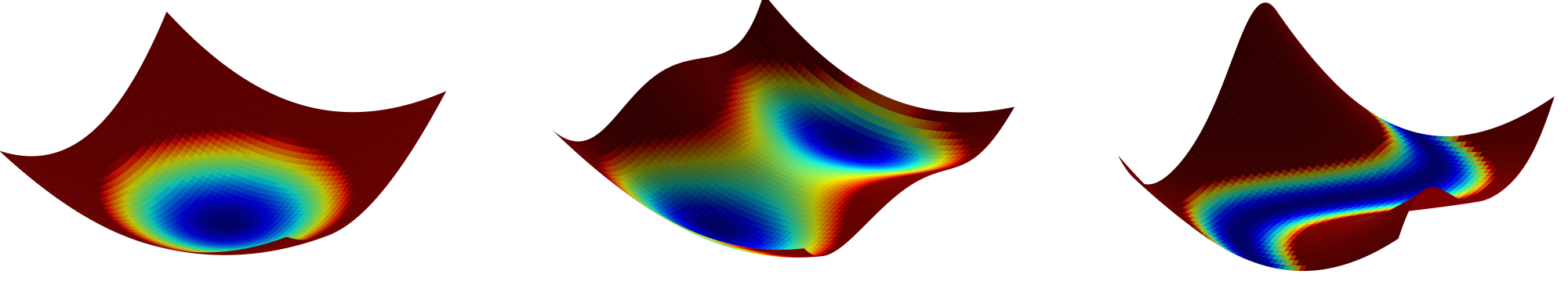}
\caption{Example loss surfaces for which our results apply. Left panel: strongly convex and convex loss surfaces (sections \ref{subsection:GD_strongly_convex} and \ref{subsubsection:convexlossesGD}). Middle panel: isolated local minima surrounded by basins of contraction (see Theorem \ref{theorem:local_robustness}). Right panel: valley of path-connected global minima (see section \ref{subsection:semi_contraction}).}
\label{fig:multi_loss}
\end{figure}

\subsection{Related Work}
Since the seminal work of \cite{bousquet2002stability}, \textit{algorithmic stability} has been used as a proxy for analyzing the generalization error of learning algorithms. There has been ample work analyzing the stability of empirical risk minimizers \citep{bousquet2002stability,mukherjee2006learning,shalev2009stochastic,shalev2010learnability}. 
A key early result in analyzing generalization in iterative optimization came from \citep{hardt2016train}, which established algorithmic stability for stochastic gradient methods. Later \citep{mou2018generalization} proved similar results for stochastic gradient Langevin dynamics. Shortly thereafter \citep{charles2018stability} showed that for loss functions satisfying certain geometrical constraints (e.g.,  Polyak-Łojasiewicz \citep{Polyak1963GradientMF}), any optimizer that converges to a global minimum is also algorithmically stable. Since then several follow-up works have analyzed the algorithmic stability of accelerated gradient methods, and the tradeoffs between optimization accuracy and algorithmic stability \citep{chen2018stability,ho2020instability,attia2021algorithmic}. The present work is similar in spirit to \cite{charles2018stability}, in the sense that we use an assumed stability property (in our case, contraction of optimizer trajectories) to derive generalization bounds for a wide class of optimizers. The following section introduces our supervised learning setting, which is the same as the one in \cite{hardt2016train}, as well as provides necessary background on algorithmic stability.

\subsection{Algorithmic Stability Background}
We consider a generic supervised learning setting where we have access to $n$ labelled examples, assumed to be drawn i.i.d from an unknown distribution $\mathcal{D}$ \citep{vapnik1999overview}. We collect these examples into a training set $S = (z_1, \dots,z_n)$.  The \textit{population risk} with respect to a loss function $\ell$ is defined as:
\[R[\boldsymbol{\theta}] = \mathbb{E}_{z \sim D} \hspace{.1cm} \ell(\boldsymbol{\theta},z)\]
where $\boldsymbol{\theta} \in \mathbb{R}^m$ describes a model. We assume that we do not know the population risk, so we use the \textit{empirical risk} as a proxy:
\[R_S[\boldsymbol{\theta}] = \frac{1}{n} \sum^{n}_{i = 1} \hspace{.1cm} \ell(\boldsymbol{\theta},z_i)\]
The difference between the population and empirical risk is denoted as the \textit{generalization error} of model $\boldsymbol{\theta}$:

\[\Delta^{gen}(\boldsymbol{\theta}) \equiv R[\boldsymbol{\theta}] - R_S[\boldsymbol{\theta}] \]
We now define the stability of an algorithm, and relate it to this generalization error. Consider an algorithm $\mathcal{A}$ which takes in $S$ and outputs a model (e.g., a parameter vector $\boldsymbol{\theta}$).

\begin{definition}[Uniform Algorithmic Stability]
An algorithm $\mathcal{A}$ is $\epsilon$-uniformly stable if for all data sets $S,S'$ such that $S$ and $S'$ differ in at most one example, we have

\begin{equation}\label{eq:eps-uniform-stable}
\sup_z \mathbb{E}_{\mathcal{A}} [\ell(\mathcal{A}(S);z) - \ell(\mathcal{A}(S');z)] \leq \epsilon
\end{equation}
\end{definition}
\noindent where the expectation is taken over the randomness of $\mathcal{A}$, if there is any. A fascinating result in learning theory states that uniform stability leads to generalization in expectation \citep{bousquet2002stability, shalev2010learnability,hardt2016train}. In particular we use Theorem 2.2 of \cite{hardt2016train}.
\begin{theorem}\label{theorem:gen_in_exp}
Let $\mathcal{A}$ be $\epsilon$-uniform stable and let $\mathbb{E}_{S,\mathcal{A}}$ denote an expectation taken over the samples $S$ and the randomness of $\mathcal{A}$. Then, $|\mathbb{E}_{S,\mathcal{A}}\Big[\Delta^{gen}(\mathcal{A}(S))\Big]| \leq \epsilon$.
\end{theorem}
If the output of $\mathcal{A}$ is some parameter vector $\boldsymbol{\theta}$ and we assume that our loss function is $L$-Lipshitz for every example $z_i$ with respect to some norm $||\cdot||$, then the difference between two trajectories of an optimizer trained on set $S$ and $S'$ can be used to bound the generalization error, because:

\begin{equation}\label{eq:lip_ass}
\mathbb{E}_{\mathcal{A}}[|\ell(\boldsymbol{\theta}_S,z)- \ell(\boldsymbol{\theta}_{S'},z)|] \leq L\, \mathbb{E}_{\mathcal{A}}||\boldsymbol{\theta}_S - \boldsymbol{\theta}_{S'}|| 
\end{equation}
Rather than only considering the Euclidean distance $||\boldsymbol{\theta}_S - \boldsymbol{\theta}_{S'}||$, in this paper we consider the \textit{geodesic distance} $d_\mathcal{M}(\boldsymbol{\theta}_S, \boldsymbol{\theta}_{S'})$ computed on a Riemannian manifold $\mathcal{M} = (\mathbb{R}^m,\mathbf{M})$ (Figure \ref{fig:geodesic}). Here $\mathbf{M}(\boldsymbol{\theta},t) \in \mathbb{R}^{m \times m}$ is the positive definite metric associated to $\mathcal{M}$. There are many optimization settings for which the geodesic distance between two points --as opposed to the Euclidean norm--is the more natural distance measure to consider \citep{amari1998natural,wensing2020beyond}. The main takeaway of this paper is that: \textit{Riemannian contraction implies generalization in supervised learning}. The details about this generalization (e.g., its dependence on the number of samples $n$ and the training time $T$) depend on the dynamical equations of the optimizer, as well as the geometry of the loss landscape, as we will see. We now provide background on nonlinear contraction analysis before stating our results.

\subsection{Nonlinear Contraction Theory Background}
Consider a state vector $\mathbf{x} \in \mathbb{R}^m$, evolving according to the continuous-time dynamics:
\begin{equation}\label{eq:dynamical_system}
\dot{\mathbf{x}} = \mathbf{f}(\mathbf{x},t)
\end{equation}
here it is assumed that all quantities are real and smooth, so any required derivative or partial derivative exists and is continuous. Then we have the following definition:

\begin{definition}[Contracting Dynamical System]\label{definition: contracting_system} 
Denote the Jacobian of \eqref{eq:dynamical_system} by $\mathbf{J} \equiv \frac{\partial \mathbf{f}}{\partial \mathbf{x}}(\mathbf{x},t)$. If there exists a symmetric positive-definite metric $\mathbf{M}(\x,t): \mathbb{R}^m\times\mathbb{R} \rightarrow \mathbb{R}^{m \times m}$ and  a scalar $\lambda > 0$ such that the following \textit{differential Lyapunov equation} is uniformly satisfied in space and time:
\begin{equation}\label{eq:cont_contraction_equation}
\dot{\mathbf{M}} +\mathbf{M}\mathbf{J} + \mathbf{J}^T\mathbf{M} \leq -2\lambda \mathbf{M}
\end{equation}
then the geodesic distance defined with respect to $\mathbf{M}$ between any two trajectories of \eqref{eq:dynamical_system} converges to zero exponentially, with rate $\lambda$, and \eqref{eq:dynamical_system} is said to be \textit{contracting}. Discrete-time contraction can be defined similarly \citep{lohmiller1998contraction}.
\end{definition}

\begin{figure}[ht]
\centering
\includegraphics[width = .5\textwidth]{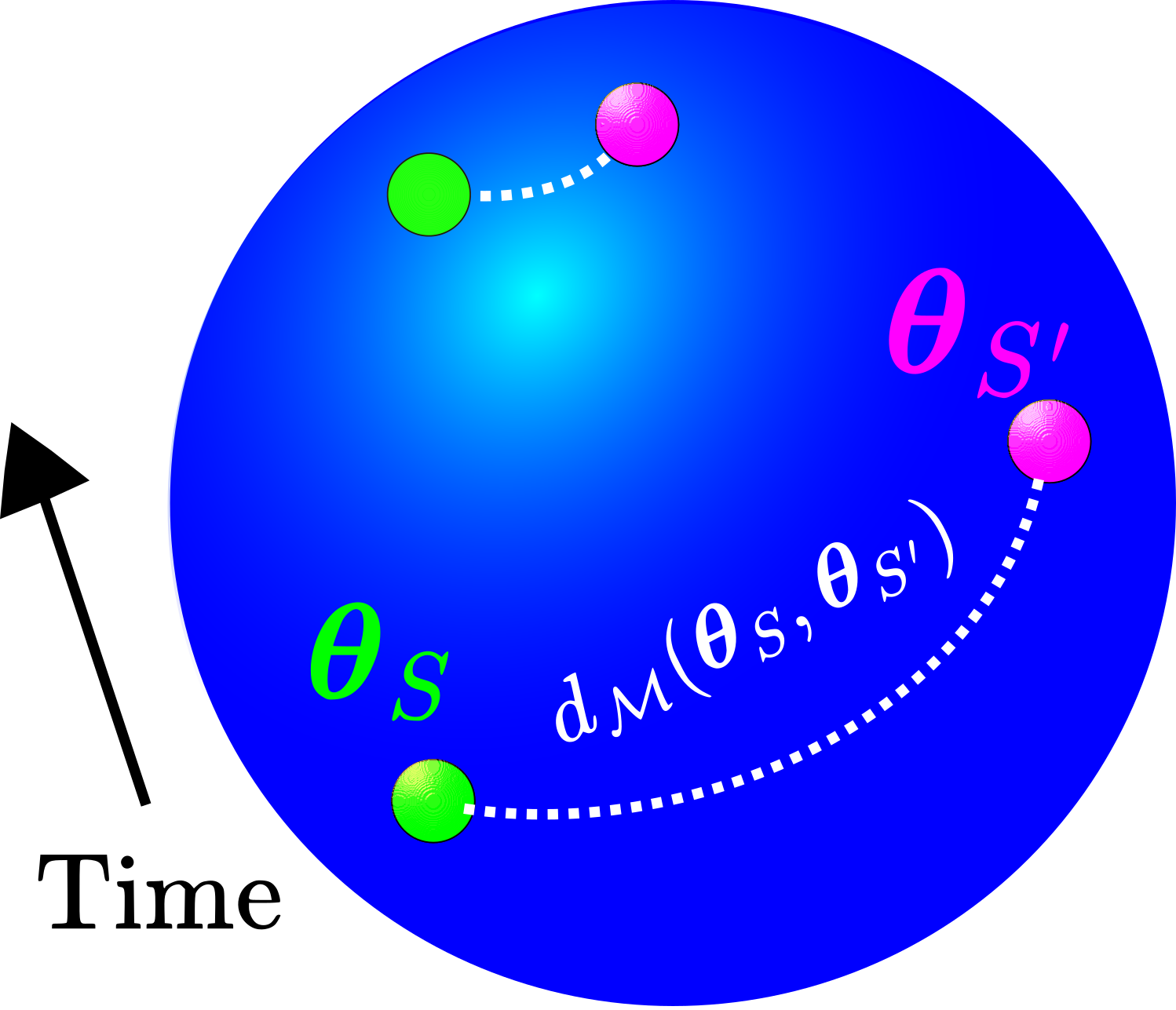}
\caption{The geodesic distance between optimizer trajectories $\boldsymbol{\theta}_S$ and $\boldsymbol{\theta}_{S'}$. If the optimizer is contracting with rate $\lambda$, this distance, denoted $d_{\mathcal{M}}(\boldsymbol{\theta}_S,\boldsymbol{\theta}_{S'})$, shrinks until the two trajectories are within a ball of radius $\mathcal{O}(\frac{1}{n \lambda})$. }
\label{fig:geodesic}
\end{figure}

\subsubsection{Robustness of Contracting Systems}
Contracting systems are robust to disturbances, in the following sense. Assume that \eqref{eq:dynamical_system} is contracting in metric $\mathbf{M} = \mathbf{T}(\mathbf{x},t)^T\mathbf{T}(\mathbf{x},t)$ with rate $\lambda$. Now consider the same dynamics as \eqref{eq:dynamical_system}, perturbed with some disturbance:

\begin{equation}\label{eq:perturbed_dynamics}
\dot{\mathbf{x}}_p = \mathbf{f}(\mathbf{x}_p,t) + \mathbf{d}(\mathbf{x}_p,t)    
\end{equation}
The geodesic distance $d_\mathcal{M}(\mathbf{x},\mathbf{x}_p)$
satisfies the \textit{differential inequality}:

\begin{equation}\label{eq:contraction_robustness_ineq}
\ddt d_\mathcal{M}(\mathbf{x},\mathbf{x}_p) + \lambda d_\mathcal{M}(\mathbf{x},\mathbf{x}_p) \leq ||\mathbf{T}(\mathbf{x},t)\mathbf{d}(\mathbf{x}_p,t)||
\end{equation}
Assuming there exists a finite constant $\distBound$ such that $||\mathbf{d}(\mathbf{x}_p,t)|| \leq \distBound$ uniformly, \eqref{eq:contraction_robustness_ineq} implies:

\begin{equation}\label{eq:contraction_robustness_equation}
R(t) \leq \chi R(0) e^{-\lambda t} + \frac{\distBound \chi}{\lambda}   
\end{equation}
where $R(t) \equiv ||\mathbf{x}(t)-\mathbf{x}_p(t)||$ and $\chi$ denotes an upper-bound on the condition number of $\mathbf{T}$. Likewise for the discrete-time dynamics contracting in some metric with rate $0 < \mu < 1$:

\[ \mathbf{x}_{t+1} = \mathbf{f}(\mathbf{x}_t,t)\]
the analogous result is:

\begin{equation}\label{eq:discrete_contraction_robustness_equation}
R(t) \leq \chi R(0)\mu^t + \frac{\distBound \chi}{(1-\mu)}   
\end{equation}
For proofs of these statements we refer the reader to \cite{lohmiller1998contraction} (section 3.7, vii) as well as \cite{del2012contraction} and \cite{zhang2021adversarially} Proposition 1 in the appendix. 

If we interpret \eqref{eq:dynamical_system} as an algorithm, then the only source of indeterminacy in this algorithm is the initial condition $\x(0)$. Therefore if \eqref{eq:dynamical_system} is always initialized within a ball of radius $C/2$ of some reference point, then \eqref{eq:contraction_robustness_equation} may be stated in expectation:

\begin{equation}\label{eq:expected_robustness_IC}
\mathbb{E}_{\mathcal{A}}[R(t)] \leq \mathbb{E}_{\mathcal{A}}[\chi R(0) e^{-\lambda t} + \frac{\distBound \chi}{\lambda} ] \leq \chi C e^{-\lambda t} + \frac{\distBound \chi}{\lambda}     
\end{equation}
where we have used the linearity of the expectation value operator, as well as the assumption $\mathbb{E}_{\mathcal{A}}[R(0)]\leq C$.

\subsubsection{Geodesics and Bounded Distortions}
To ensure that our results are coordinate-free, we show that the `distortion factor' between the geodesic distances computed along two different manifolds $\mathcal{M}_1$ and $\mathcal{M}_2$ is uniformly bounded. The practical implication is that geodesic distances as measured in two different metrics can differ by no more than a constant factor, which precludes any situation where a system is stable in one metric (geodesic distances between trajectories shrink to zero) and not stable in another metric (geodesic distances do not shrink to zero). 

\begin{theorem}\label{theorem:bounded_distortion}
Consider two Riemannian metrics $\mathbf{M}_1(\mathbf{x},t)$ and $\mathbf{M}_2(\mathbf{x},t)$ satisfying:

\[ \mu_1 \mathbf{I}\preceq \mathbf{M}_1(\mathbf{x},t) \preceq L_1\mathbf{I} \hspace{1cm}\text{and}\hspace{1cm} \mu_2 \mathbf{I}\preceq \mathbf{M}_2(\mathbf{x},t) \preceq L_2\mathbf{I}\]
with $\mu_i,L_i > 0$. Then the corresponding geodesic distances evaluated between two points, $\mathbf{x}$ and $\mathbf{y}$,  satisfy the bound:

\[ \sqrt{\frac{\mu_1}{L_2}}\leq \frac{d_{\mathcal{M}_1}(\mathbf{x},\mathbf{y})}{d_{\mathcal{M}_2}(\mathbf{x},\mathbf{y})} \leq \sqrt{\frac{L_1}{\mu_2}}\]

\end{theorem}

\begin{proof}
 The geodesic distances corresponding to these two metrics, evaluated between two points, $\mathbf{x}$ and $\mathbf{y}$  each are bounded in terms of the Euclidean norm as follows (see e.g \citep{boffi2021regret} Proposition D.2):

\[ \sqrt{\mu_1}||\mathbf{x}-\mathbf{y}||_2 \leq d_{\mathcal{M}_1}(\mathbf{x},\mathbf{y}) \leq \sqrt{L_1}||\mathbf{x}-\mathbf{y}||_2 \hspace{0.1cm}\text{and}\hspace{0.1cm} \sqrt{\mu_2}||\mathbf{x}-\mathbf{y}||_2 \leq d_{\mathcal{M}_2}(\mathbf{x},\mathbf{y}) \leq \sqrt{L_2}||\mathbf{x}-\mathbf{y}||_2\]
which implies that the distortion between $d_{\mathcal{M}_1}(\mathbf{x},\mathbf{y})$ and $d_{\mathcal{M}_2}(\mathbf{x},\mathbf{y})$ (as measured by their ratio) is bounded as follows:

\[ \sqrt{\frac{\mu_1}{L_2}}\leq \frac{d_{\mathcal{M}_1}(\mathbf{x},\mathbf{y})}{d_{\mathcal{M}_2}(\mathbf{x},\mathbf{y})} \leq \sqrt{\frac{L_1}{\mu_2}}\]
\end{proof}

\section{Main Results} 

\subsection{Contracting Optimizers are Algorithmically Stable}
In this section we prove our main result for continuous-time optimizers using the entire training batch. We start with this case because it is the simplest. Later on, we provide the same result for stochastic, discrete-time optimizers such as mini-batch stochastic gradient descent. We assume that our parameter update is \textit{sum-separable} with respect to training set $S$:

\begin{equation}\label{eq:learning_dynamics_contraction_S}
\thetaDot_S = \mathbf{G}(\boldsymbol{\theta}_S,S) = \frac{1}{n}\sum^n_{i = 1} \mathbf{g}(\boldsymbol{\theta}_S,z_i)     
\end{equation}
 In this case the output of algorithm $\mathcal{A}(S)$ is the vector $\boldsymbol{\theta}_S$ obtained by simulating \eqref{eq:learning_dynamics_contraction_S} for time $t$. We also assume that $||\mathbf{g}|| \leq \xi$, for some constant $\xi$. If we interpret $\mathbf{g}$ as the gradient of some loss $\ell$, then this corresponds to assuming that $\ell$ is $\xi$-Lipschitz. Finally, we assume that the optimizer is always initialized--perhaps randomly-- within a ball of radius $C/2$ around some reference point. Now consider the same parameter update with respect to training set $S'$, which differs from $S$ in one example:

\begin{equation}\label{eq:learning_dynamics_contraction_S'}
\thetaDot_{S'} = \mathbf{G}(\boldsymbol{\theta}_{S'},{S'}) 
\end{equation}
We can now state our first main result:

\begin{theorem}\label{theorem:Contraction_Implies_AS}[Contraction Implies Algorithmic Stability]
If the dynamics \eqref{eq:learning_dynamics_contraction_S} are contracting in metric $\mathbf{M} = \mathbf{T}(\boldsymbol{\theta},t)^T\mathbf{T}(\boldsymbol{\theta},t)$ with rate $\lambda$, then $\mathcal{A}$ is uniformly $\epsilon$-stable, with:

\begin{equation}\label{eq:robustness_gen_result}
\epsilon \leq \chi L e^{-\lambda t}C + \frac{2\chi L \xi}{\lambda n}    
\end{equation}
where $\chi$ denotes a uniform upper-bound on the condition number of $\mathbf{T}(\boldsymbol{\theta},t)$. Going forward we refer to $\epsilon_{stab} \equiv  \frac{2\chi L \xi}{\lambda n}$.

\end{theorem}

\begin{proof}

The goal is to write \eqref{eq:learning_dynamics_contraction_S'} as a perturbed version of \eqref{eq:learning_dynamics_contraction_S} and then apply the robustness property of contracting systems to yield the result. Note that:

\[\thetaDot_{S'} = \frac{1}{n}\sum^n_{i = 1} \mathbf{g}(\boldsymbol{\theta}_{S'},z_i) - \frac{1}{n}(\mathbf{g}(\boldsymbol{\theta}_{S'},z_i) - \mathbf{g}(\boldsymbol{\theta}_{S'},z'_i)) \]
where we have just subtracted out the term involving $z_i$ from the sum, and added in the replacement term $z'_i$. This may be viewed as a perturbed version of \eqref{eq:learning_dynamics_contraction_S}, with disturbance:

\[ ||\mathbf{d}(\boldsymbol{\theta}_{S'},z_i,z'_i,n)|| = || \frac{1}{n}(\mathbf{g}(\boldsymbol{\theta}_{S'},z_i) - \mathbf{g}(\boldsymbol{\theta}_{S'},z'_i))|| \leq \frac{2\xi}{n} = \distBound \]
Plugging $\distBound$ into \eqref{eq:expected_robustness_IC}, multiplying through by $L$ because of \eqref{eq:lip_ass}, and taking the expectation $\mathbb{E}_{\mathcal{A}}$ to produce $R(0)$ yields the result. 
\end{proof}

\begin{remark}[Leave-One-Out Stability]
As pointed out in \citep{bousquet2020sharper}, for interpolation algorithms (such as, e.g., the highly overparameterized searches common in deep learning) it is more meaningful to analyze \emph{leave-one-out} stability, rather than replace-one stability as we just did. In this case the same dynamical robustness argument applies immediately, so that $\distBound$ and therefore $\epsilon_{stab}$ are just reduced by a factor of two. 
\end{remark} 

\begin{remark}[Generalization with High Probability]
A well-known limitation of using algorithmic stability to derive generalization bounds is that the bounds only hold in expectation. However, one can use Chebyshev's inequality to derive generalization bounds that hold with high probability \citep{bousquet2002stability, elisseeff2005stability, feldman2019high,bousquet2020sharper}. It is well known that these bounds are tight in the case when algorithmic stability scales with $1/n$, see e.g., Theorem 12 and Remark 13 in \citep{bousquet2002stability}.  Theorem \ref{theorem:Contraction_Implies_AS} shows that this is the case for contracting optimizers. In Section~\ref{stochastic}, Theorem \ref{theorem:Contraction_Implies_AS_discrete} will show that this $1/n$ scaling also holds for the stochastic optimization case. 

\end{remark} 

\begin{remark}[Scaling Dynamics Does not Change Generalization Rate]
Note that if we `speed up' the dynamics \eqref{eq:learning_dynamics_contraction_S} by some factor $\beta > 0$:
\[\mathbf{G}(\boldsymbol{\theta}_S,S) \rightarrow \beta\mathbf{G}(\boldsymbol{\theta}_S,S) \]
one might intuitively expect the contraction rate to be scaled by $\beta$ as well ($\lambda \rightarrow \beta \lambda$), which would allow an arbitrary increase of the rate of generalization in  \eqref{eq:robustness_gen_result} by simply increasing $\beta$. Note however that this is prevented by the presence of $\xi$ in \eqref{eq:robustness_gen_result}, which is also scaled by $\beta$. The $\beta$ terms in the numerator and denominator therefore cancel out, leaving $\epsilon_{stab}$ unchanged.
\end{remark}

\begin{remark}[Lipschitz Assumption]\label{remark: lip_ass_caveat}
As pointed out in \cite{hardt2016train}, there are cases where $L$ as defined in \eqref{eq:lip_ass} may not exist. For example strongly convex functions have unbounded gradients on $\mathbb{R}^m$. In this case we will overload the symbol $L$ to be:
\[L = \sup_{\boldsymbol{\theta} \in \Omega}\sup_{z} ||\nabla \ell(\boldsymbol{\theta},z)||_2\]
where $\Omega$ denotes a convex, compact set over which we are optimizing. For contracting optimizers and $\beta$-smooth loss functions, $L$ is always finite. This is because contraction precludes finite escape \citep{lohmiller1998contraction}, and therefore $\text{diameter}({\Omega})$ is well defined and we have $L \leq \beta \text{diameter}({\Omega})$. \end{remark}

\begin{figure}[ht]
\centering
\includegraphics[width = \textwidth]{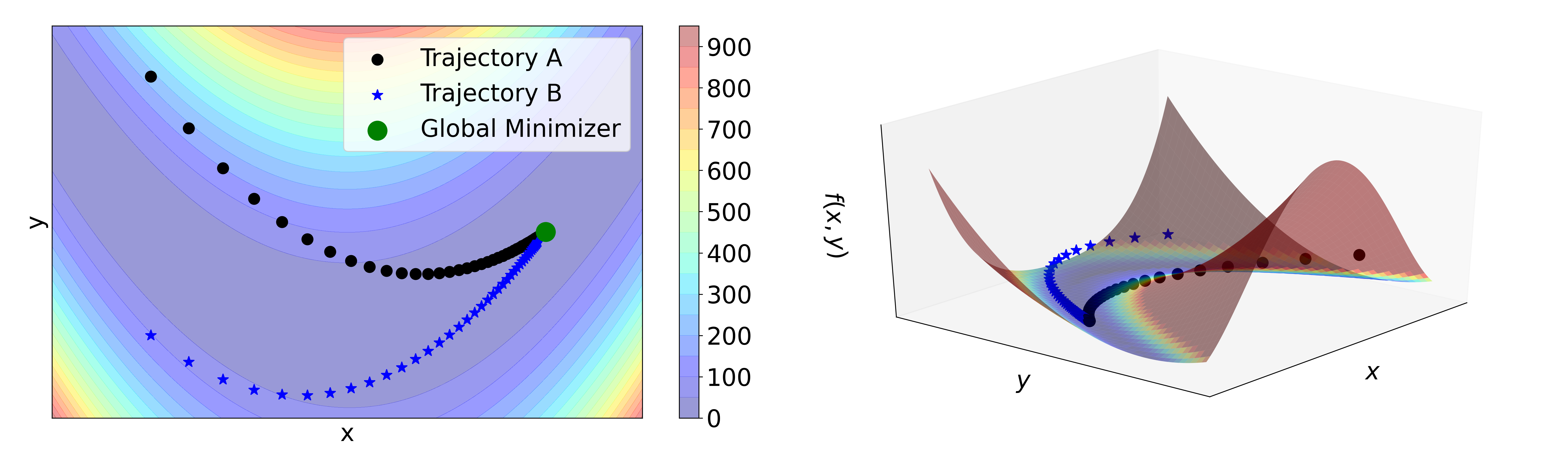}
\caption{\textit{Left subplot)} Two trajectories of a contracting optimizer, seeded from two different initial conditions, evolving over a non-convex loss surface (Rosenbrock function, $f(x,y) =100(x^2 - y)^2 + (x-1)^2$). Both exponentially converge to the global minimizer of the function. Trajectories superimposed over contour plot of the loss surface. \textit{Right subplot)} A different view of the same optimization process, more clearly displaying the non-convexity of the loss surface.}
\label{fig:rosenbrock}
\end{figure}

\subsection{Stochastic, Contracting Optimizers are Algorithmically Stable}\label{stochastic}
In this section we show that a variant of Theorem \ref{theorem:Contraction_Implies_AS} holds for stochastic, discrete-time optimizers (for example mini-batch stochastic gradient descent). Consider the iterative optimizer:

\begin{equation}\label{eq:discrete_update}
\boldsymbol{\theta}^{S}_{t+1} = \frac{1}{b} \sum_{i = 1}^b \mathbf{g}(\boldsymbol{\theta}^{S}_t,z_i) 
\end{equation}
where $ 1 \leq b \leq n$ is the size of the mini-batch and $z_i$ are samples drawn randomly from set $S$. As before we assume that $\mathbf{g}$ is smooth and bounded as $||\mathbf{g}|| \leq \xi$. Since \eqref{eq:discrete_update} defines a discrete-time, random dynamical system \citep{tabareau2013contraction} we have to define what we mean by `contraction'. In particular we will rely on an assumption of `contraction in expectation', by which we mean the following. Consider two instantiations of the same discrete-time, random dynamical system:

\[\x_{t+1} = \f(\x_t,t,\Gamma) \hspace{1cm} \text{and} \hspace{1cm} \y_{t+1} = \f(\y_t,t,\Gamma) \]
where $\Gamma$ denotes a \textit{particular realization} of a stochastic process which is the same for both $\x$ and $\y$. In our case, this stochasticity stems from the random sampling of training set datapoints to form a mini-batch. We will say that this system is \textit{contracting in expectation} if for a sequence of metrics $\mathbf{M}_{0}, \dots, \mathbf{M}_{t}$ we have:

\[ \mathbb{E}_\mathcal{A}[d_{\mathcal{M}_{t+1}}(\f(\x_t,t,\Gamma),\f(\y_t,t,\Gamma))] \leq \mu\mathbb{E}_\mathcal{A}[d_{\mathcal{M}_t}(\x_t,\y_t)] \]
where $0 < \mu < 1$ and each metric is bounded $M_{min}\mathbf{I} \preceq \mathbf{M}_i \preceq M_{max}\mathbf{I}$. We can now state the following theorem:

\begin{theorem}\label{theorem:Contraction_Implies_AS_discrete}[Contraction Implies Algorithmic Stability (Stochastic, Discrete)]
Assume \eqref{eq:discrete_update} is contracting in expectation, as defined above. In this case $\mathcal{A}$ is uniformly $\epsilon$-stable with bound: 

\begin{equation}\label{eq:robustness_gen_result_discrete}
\epsilon \leq  L \chi C \mu^t + \frac{2\chi L \xi}{(1-\mu) n}    
\end{equation}
\end{theorem}

\begin{proof}
 For every time $t$ we randomly sample $b$ indices ${i_1,\cdots,i_b}$. Using these indices we select datapoints ${z_i, \cdots,z_{i_b}}$ from $S$ and $S'$ to update $\boldsymbol{\theta}_t^S$ and $\boldsymbol{\theta}_t^{S'}$ respectively. At every time $t$ there are two possibilities. Either we do not draw the replaced element $z'$ or we do. Denote these events $A$ and $B$, respectively (Figure \ref{fig:mini-batches}). We have $P(A) = 1-\frac{b}{n}$ and $P(B) = \frac{b}{n}$. If  event $A$ occurs, then by assumption we expect the geodesic distance to shrink:
\begin{equation}\label{eq:no_draw_replaced_stoch}
 \mathbb{E}_\mathcal{A}[d_{\mathcal{M}_{t+1}}(\boldsymbol{\theta}^{S}_{t+1},\boldsymbol{\theta}^{S'}_{t+1}) | A] \leq \mu\mathbb{E}_\mathcal{A}[d_{\mathcal{M}_t}(\boldsymbol{\theta}^{S}_{t},\boldsymbol{\theta}^{S'}_{t}) | A]     
\end{equation}

where $\mathbb{E}[\cdot|A]$ denotes the conditional expectation given event $A$.
However, if the replaced element is drawn (i.e., event $B$ occurs) then we have:

\[\boldsymbol{\theta}^{S}_{t+1} =  \frac{1}{b} \sum_{i = 1}^b \mathbf{g}(\boldsymbol{\theta}^{S}_t,z_i) \equiv \hat{\mathbf{G}}(\boldsymbol{\theta}_t^{S}) \]

\[\boldsymbol{\theta}^{S'}_{t+1} =  \hat{\mathbf{G}}(\boldsymbol{\theta}_t^{S'}) + \mathbf{d}(\boldsymbol{\theta}_t^{S'}) \]
where $\mathbf{d}(\boldsymbol{\theta}_t^{S'}) = \frac{1}{b}(\mathbf{g}(\boldsymbol{\theta}_t^{S'},z'_i)-\mathbf{g}(\boldsymbol{\theta}_t^{S'},z_i))$. As in Theorem \eqref{theorem:Contraction_Implies_AS}, we have written the update for $\boldsymbol{\theta}^{S'}$ as a `perturbed' version of the update for $\boldsymbol{\theta}^{S}$. We will now derive an analogous robustness result, and then use the linearity of expectation to bound the overall geodesic distance. Note that:
\begin{align*}
    d_{\mathcal{M}_{t+1}}(\boldsymbol{\theta}^{S}_{t+1},\boldsymbol{\theta}^{S'}_{t+1}) &= d_{\mathcal{M}_{t+1}}(\hat{\mathbf{G}}(\boldsymbol{\theta}^{S}_{t}),\hat{\mathbf{G}}(\boldsymbol{\theta}^{S'}_{t}) + \mathbf{d}(\boldsymbol{\theta}_t^{S'})) \\
    &\leq d_{\mathcal{M}_{t+1}}(\hat{\mathbf{G}}(\boldsymbol{\theta}^{S}_{t}),\hat{\mathbf{G}}(\boldsymbol{\theta}^{S'}_{t})) + d_{\mathcal{M}_{t+1}}(\hat{\mathbf{G}}(\boldsymbol{\theta}^{S'}_{t}),\hat{\mathbf{G}}(\boldsymbol{\theta}^{S'}_{t}) + \mathbf{d}(\boldsymbol{\theta}_t^{S'})) \\
    &\leq d_{\mathcal{M}_{t+1}}(\hat{\mathbf{G}}(\boldsymbol{\theta}^{S}_{t}),\hat{\mathbf{G}}(\boldsymbol{\theta}^{S'}_{t})) + \sqrt{M_{max}}\frac{2\xi}{b}    
\end{align*}
where the first inequality comes from the triangle inequality and the second comes from the boundedness of $\mathbf{d}(\boldsymbol{\theta}_t^{S'})$ and the metric distortion bound in Theorem \ref{theorem:bounded_distortion}. Now applying the assumption of contraction in expectation we get:

\begin{equation}\label{eq:draw_replaced_stoch}
\begin{split}
    \mathbb{E}_{\mathcal{A}}[d_{\mathcal{M}_{t+1}}(\boldsymbol{\theta}^{S}_{t+1},\boldsymbol{\theta}^{S'}_{t+1})|B] \leq \mathbb{E}_{\mathcal{A}}[d_{\mathcal{M}_{t+1}}(\hat{\mathbf{G}}(\boldsymbol{\theta}^{S}_{t}),\hat{\mathbf{G}}(\boldsymbol{\theta}^{S'}_{t}))|B] + \sqrt{M_{max}}\frac{2\xi}{b} \\
    \leq \mu \mathbb{E}_{\mathcal{A}}[d_{\mathcal{M}_{t}}(\boldsymbol{\theta}^{S}_{t},\boldsymbol{\theta}^{S'}_{t})|B] + \sqrt{M_{max}}\frac{2\xi}{b} 
\end{split}    
\end{equation}
We can now use the linearity of the expectation operator to bound the geodesic distance, and then use the metric distortion result to bound the Euclidean distance:
\begin{align*}
\mathbb{E}_{\mathcal{A}}[d_{\mathcal{M}_{t+1}}(\boldsymbol{\theta}^{S}_{t+1},\boldsymbol{\theta}^{S'}_{t+1})] = \mathbb{E}_{\mathcal{A}}[d_{\mathcal{M}_{t+1}}(\boldsymbol{\theta}^{S}_{t+1},\boldsymbol{\theta}^{S'}_{t+1})|A]P(A) + \mathbb{E}_{\mathcal{A}}[d_{\mathcal{M}_{t+1}}(\boldsymbol{\theta}^{S}_{t+1},\boldsymbol{\theta}^{S'}_{t+1})|B]P(B) \\
\leq \mu\mathbb{E}_\mathcal{A}[d_{\mathcal{M}_t}(\boldsymbol{\theta}^{S}_{t},\boldsymbol{\theta}^{S'}_{t})](1-\frac{b}{n}) + (\mu \mathbb{E}_{\mathcal{A}}[d_{\mathcal{M}_{t}}(\boldsymbol{\theta}^{S}_{t},\boldsymbol{\theta}^{S'}_{t})] + \sqrt{M_{max}}\frac{2\xi}{b})\frac{b}{n} \\
= \mu\mathbb{E}_\mathcal{A}[d_{\mathcal{M}_t}(\boldsymbol{\theta}^{S}_{t},\boldsymbol{\theta}^{S'}_{t})] + \sqrt{M_{max}}\frac{2\xi}{n}
\end{align*}
Using the metric distortion bounds and unravelling the recursion yields:

\[ \mathbb{E}_{\mathcal{A}}[d(\boldsymbol{\theta}^{S}_{t},\boldsymbol{\theta}^{S'}_{t})] \leq \chi \mu^tC + \frac{2 \chi \xi}{(1-\mu)n } \]
Where $\chi = \sqrt{\frac{M_{max}}{M_{min}}}$ has again come from the metric distortion bound. Multiplying through by $L$, we have that:
\[\epsilon_{stab} = \frac{2L\chi \xi}{n(1-\mu)} \]
which is the same result as the continuous-time case, expect that $\lambda \rightarrow (1-\mu)$. 
\end{proof}

\begin{figure}[ht]
\centering
\includegraphics[width = \textwidth]{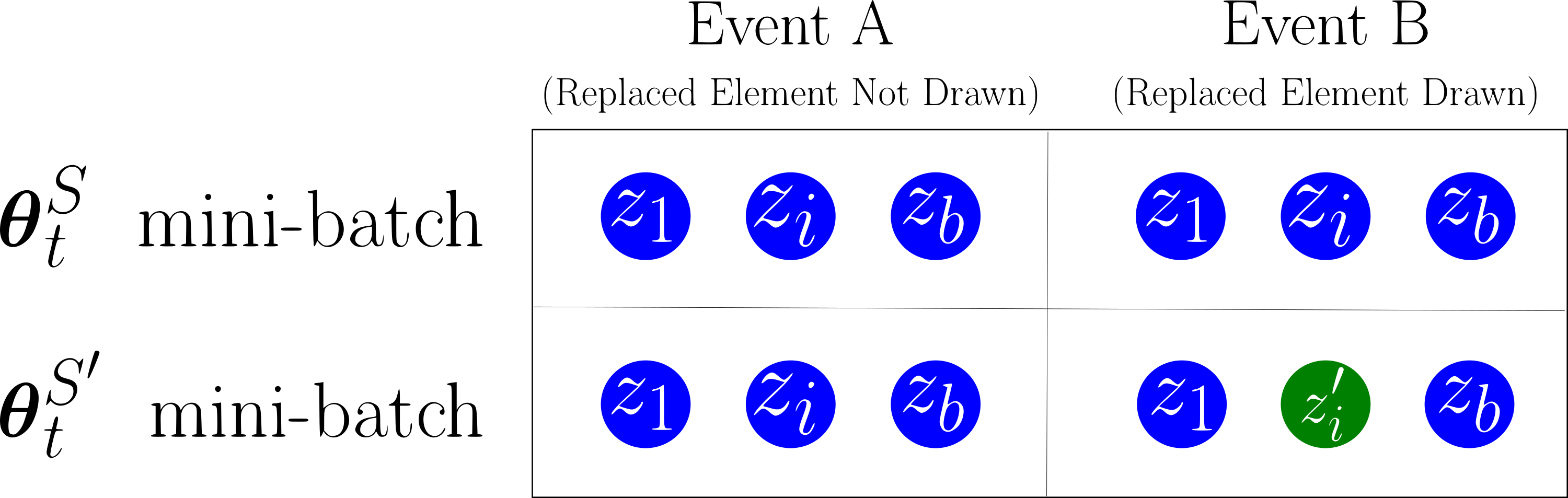}
\caption{Illustration of the two cases for updating with mini-batches. At each time $t$, we randomly sample $b$ indices between $1$ and $n$. Then we draw the corresponding data-points from sets $S$ and $S'$ to form the mini-batches used to update $\boldsymbol{\theta}^S_t$ and $\boldsymbol{\theta}^{S'}_t$ respectively. In Event A (left column), the index of the replaced element is not selected, and therefore the datapoints used to update $\boldsymbol{\theta}^S_t$ and $\boldsymbol{\theta}^{S'}_t$ are the same. In Event B, the index of the replaced element is selected, and so the datapoints used to perform the update are different. }
\label{fig:mini-batches}
\end{figure}

\section{Examples}
\subsection{Preconditioned Gradient Descent On Strongly Convex Loss Functions} 
\label{subsection:GD_strongly_convex}
In this example we show that our theory reproduces known stability bounds for gradient descent on strongly convex losses. To illustrate the role of the contraction metric, we consider \textit{preconditioned} gradient descent. Consider this descent over an empirical loss function which is $\gamma$-strongly convex with respect to a parameter vector $\boldsymbol{\theta}$: 

\[\thetaDot = -\mathbf{P}^{-1}\nabla \mathcal{L} \]
where $\mathbf{P}$ is a positive-definite and symmetric matrix. Denote the largest and smallest eigenvalues of $\mathbf{P}$ as $p_{max}$ and $p_{min}$, respectively. The Jacobian of this system is:

\[\mathbf{J} = \frac{\partial \thetaDot}{\partial \boldsymbol{\theta}} = -\mathbf{P}^{-1}\nabla^2 \mathcal{L}\]
Picking the metric $\mathbf{M} = \mathbf{P}$, we see that:

\[\mathbf{P}\mathbf{J} + \mathbf{J}^T\mathbf{P} = -2\nabla^2 \mathcal{L} \leq -2\gamma \mathbf{I} \leq -\frac{2\gamma}{p_{max}}\mathbf{P} \]
and thus the system is contracting in metric $\mathbf{P}$ with rate $\lambda = \gamma/p_{max}$. Our algorithmic stability bound is therefore:

\[\epsilon_{stab} =  \sqrt{\frac{p_{max}^3}{p_{min}}} \frac{2L^2}{\gamma n} \]
Where $L$ is given by \eqref{remark: lip_ass_caveat}. Note that in the case of regular gradient descent, without preconditioning (i.e., $\mathbf{P} = \mathbf{I}$) the above analysis shows that $\lambda = \gamma$ and $\chi = 1$. Plugging these numbers into equation \eqref{eq:robustness_gen_result} yields the following:

\[\epsilon_{stab} =  \frac{2 L^2}{\gamma n} \]
which is precisely the result of Theorem 3.9 in \citep{hardt2016train}.

\begin{remark}[Natural Gradient on Geodesically Strongly Convex Losses]\label{remark:natural-gradient-remark}
Natural gradients are a popular way to incorporate geometric information about the loss surface into gradient-based optimization techniques \citep{amari1998natural,zhang2019fast}.
An equivalence between g-Strong Convexity and global contraction of natural gradient flows was is given in (Theorem 1, \citep{wensing2020beyond}). That is, the optimizer dynamics:

\[\thetaDot = - \mathbf{M}(\boldsymbol{\theta})^{-1}\nabla \mathcal{L}(\boldsymbol{\theta}) \]
are globally contracting if and only if $\mathcal{L}(\boldsymbol{\theta})$ is geodesically strongly convex over $\mathcal{M}$. In this case Theorem \ref{theorem:Contraction_Implies_AS} of the present work applies immediately, in precisely the same fashion as the preceding subsection.
\end{remark}

\subsection{Picking the Best Metric for Kernel Regression}
For constant metrics, our stability bound $\epsilon_{stab} \sim \frac{\chi}{\lambda}$ depends on the condition number of the contraction metric (specifically its square root) to the contraction rate \textit{measured in that metric}. Different metrics yield different $\epsilon_{stab}$, so it is natural to ask whether an `optimal' metric $\mathbf{M}_{optimal}$ exists, such that:

\[\epsilon_{stab}(\mathbf{M}_{optimal}) \leq \epsilon_{stab}(\mathbf{M}) \]
While finding such a metric is in general not easy to do, we show that it is possible in the case of gradient descent for kernel ridge regression \citep{shawe2004kernel}. Kernel methods (which are inherently linear) can be used to derive insights into nonlinear systems such deep neural networks \citep{jacot2018neural,lee2019wide,fort2020deep,canatar2021spectral}. Without loss of generality, we assume an element-wise feature map such that for a matrix $\mathbf{X} \in \mathbb{R}^{q \times z}$, the matrix $\phi(\mathbf{X})\in \mathbb{R}^{q \times  z}$ satisfies $\phi(\mathbf{X})_{ij} = \phi(\mathbf{X}_{ij})$. The squared-loss for kernel ridge-regression is:

\[\mathcal{L} = \frac{1}{2n}\sum_{i = 1}^n (\phi(\mathbf{x}_i)\mathbf{w} - y_i)^2  + \frac{\alpha}{2} ||\mathbf{w}||^2\]
where the $\phi(\mathbf{x}_i)$ are feature row vectors, $\mathbf{w} \in \mathbb{R}^m$ is the linear model to be learned, and the $y_i \in \mathbb{R}$ are target labels. The parameter $\alpha> 0$ is the regularization parameter. Under gradient descent $\dot{\mathbf{w}} = -\nabla\mathcal{L}$ the Jacobian of the optimizer dynamics is:

\[\frac{\partial \dot{\mathbf{w}} }{\partial \mathbf{w}} = \mathbf{J} = -(\mathbf{G} + \alpha\mathbf{I}) \]
where $\mathbf{G} \equiv \frac{1}{n}\phi(\mathbf{X})^T\phi(\mathbf{X})$ and $\mathbf{X} \in \mathbb{R}^{n \times d}$ is a constant matrix with $\x_i$ as the $i^{th}$ row. The Jacobian $\mathbf{J}$ is symmetric, constant and negative-definite. Thus, the optimizer is contracting in the identity metric with rate $\lambda_I = \lambda_{min}(\mathbf{G}) + \alpha$, where $\lambda_{min}(\cdot)$ denotes the smallest eigenvalue. We now prove the following:

\begin{theorem}
For kernel ridge regression, the algorithmic stability bound $\epsilon_{stab}$ is minimized for $\mathbf{M} = \mathbf{I}$.
\end{theorem}

\begin{proof}
Recall that for an arbitrary, constant metric we are looking for a positive-definite symmetric $\mathbf{Q}$ such that:
\[  \mathbf{M}\mathbf{J} + \mathbf{J}\mathbf{M} = -\mathbf{Q} \leq -2\lambda \mathbf{M} \]
Ignoring $\chi$ for a moment, we can ask: out of the set of all possible metrics, is there a metric that yields the \textit{largest} contraction rate $\lambda$?  An interesting result from linear dynamical systems theory is that the answer is in fact yes. While there can be many metrics for linear systems that give the largest possible $\lambda$, one can always be found from setting $\mathbf{Q}=\mathbf{I}$ and solving for $\mathbf{M}$ (see, e.g., section 3.5.5 in \citep{slotine1991applied}). Since $\mathbf{J}$ is symmetric, in our case this metric corresponds to the diagonalizing metric:

\[\mathbf{M}_{largest} = \frac{1}{2}\mathbf{J}^{-1} =  \frac{1}{2}(\mathbf{G + \alpha\mathbf{I}})^{-1} \]
The contraction rate $\lambda_{largest}$ corresponding to this metric is:
\[\lambda_{largest} = \frac{1}{2}\frac{1}{\lambda_{max}(\mathbf{M}_{largest})} = \lambda_{min}(\mathbf{G}) + \alpha \]
which is precisely the same contraction rate as measured in the identity metric. Thus $\lambda_I = \lambda_{largest}$. Now we simply use the fact that $\chi_I = 1 \leq \chi_M$ for any metric. Since $\mathbf{M} = \mathbf{I}$ corresponds to the largest possible $\lambda$ and the smallest possible $\chi = 1$, the ratio of $\chi$ to $\lambda$ is minimal over all possible $\mathbf{M}$ when $\mathbf{M} = \mathbf{I}$. Thus:

\[\epsilon_{stab}(\mathbf{I}) \leq \epsilon_{stab}(\mathbf{M}) \]
\end{proof}
This result is illustrated in Figure \ref{fig:opt_rate}. To create this plot we generated a random $\mathbf{G} \in \mathbb{R}^{3 \times 3}$. Then we generated random $\mathbf{Q}$ and solved the Lyapunov equation for $\mathbf{M}$ using an implementation of the Bartels-Stewart algorithm in SciPy \citep{10.1145/361573.361582,2020SciPy-NMeth}. In addition to these random $\mathbf{Q}$, we also set $\mathbf{Q} = \mathbf{I}$ to obtain the $\mathbf{M}$ corresponding to the largest $\lambda$. For each of these $\mathbf{Q}$ and $\mathbf{M}$ pairs, $\lambda_M$ is given by $\lambda_M = \frac{1}{2}\frac{\lambda_{min}(\mathbf{Q})}{\lambda_{max}(\mathbf{M})}$ \citep{slotine1991applied}. We computed $\mathbf{T}$ via a Cholesky decomposition (also using SciPy) and then performed a singular value decomposition to obtain $\chi$.  One interpretation of this result is: there is no `better' coordinate system. That is, there is no coordinate transformation we could perform on the state vector $\mathbf{w}$ which would give us a tighter algorithmic stability bounds. This is because a constant metric $\mathbf{M} = \mathbf{T}^T\mathbf{T}$ corresponds to the coordinate change $\mathbf{w} \rightarrow \mathbf{T}\mathbf{w}$. 
\begin{figure}[ht]
\centering
\includegraphics[width = \textwidth]{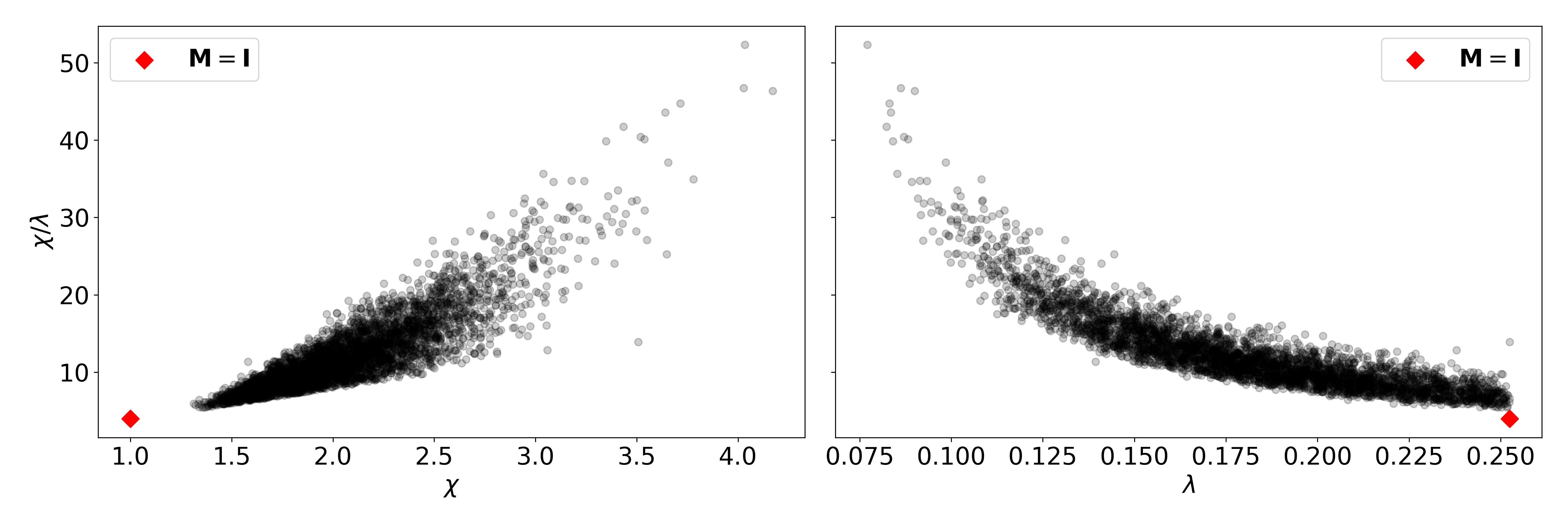}
\caption{For a fixed $\mathbf{G} \in \mathbb{R}^{3 \times 3}$, we randomly generate $\mathbf{Q}$ and solve for $\mathbf{M}$. We then calculate the ratio $\chi/\lambda_M$ (the only metric-dependent terms in our algorithmic stability bound). We repeat this procedure 4000 times. \textit{Left subplot)} The ratio $\chi/\lambda_M$ plotted against $\chi$. \textit{Right subplot)} The ratio $\chi/\lambda_M$ bound plotted against $\lambda_M$. Details are in main text.
These plots illustrate our theoretical result: that the identity metric gives the optimal (i.e smallest) algorithmic stability bound. They also illustrates the reason: the identity metric simultaneously obtains the smallest condition number and the largest contraction rate, thus minimizing the ratio of the former to the latter.} 
\label{fig:opt_rate}
\end{figure}

\section{Weaker Notions of Stability}
Contraction imposes a strong condition on optimizer trajectories: they must converge towards one another exponentially. Such convergence can be expected around isolated local or global minima, as discussed above. However in modern machine learning, one often observes optimizer trajectories which converge towards a common basin of low/zero loss, where minima may lie among a low-dimensional manifold \citep{garipov2018loss,draxler2018essentially,fort2020deep,liu2021loss}. To accommodate these cases, we now discuss several weaker notions of contraction--specifically local contraction, semi-contraction and partial contraction--which also yield `well-behaved' algorithmic stability bounds. 

\subsection{Loss Surfaces with Many Local Minima}
A contraction region (i.e., a region of state space that satisfies Definition \ref{definition: contracting_system}) for an autonomous system contains at most one equilibrium point \citep{lohmiller1998contraction}. From this it follows that gradient descent over a loss surface with many local equilibria cannot be globally contracting. Fortunately, if Definition \ref{definition: contracting_system} holds within a subset of state-space, and additionally the system can be show to remain in that subset for all time (i.e., the subset is forward invariant), then that system is locally contracting. This motivates the following general result, as well as optimization-specific remark. 

\begin{theorem}\label{theorem:local_robustness}
Consider the system \eqref{eq:dynamical_system} initialized inside an inner Euclidean ball of radius $b$, which is fully contained within a outer contraction region (which we also assume without loss of generality to be a Euclidean ball) of radius $B > b$. Assume that \eqref{eq:dynamical_system} stays within the inner ball for all time. Now consider the perturbed dynamics \eqref{eq:perturbed_dynamics}. If $B \ge b(\chi + 1) + \frac{\chi \distBound}{\lambda}$ then \eqref{eq:perturbed_dynamics} stays within the outer contraction region for all time, and the robustness result \eqref{eq:contraction_robustness_equation} holds. 
\end{theorem}

\begin{proof}
By \eqref{eq:contraction_robustness_equation}, the perturbed trajectory will be at most distance $\chi b + \frac{\chi \distBound}{\lambda}$ from the unperturbed trajectory. Since the unperturbed trajectory is always contained within a ball of radius $b$, this implies the perturbed trajectory is also contained in a ball of radius $b + \chi b + \frac{\chi \distBound}{\lambda}$, assuming it stays in a contraction region. To ensure that it does in fact stay within a contraction region, we must have that $B \ge b(\chi + 1) + \frac{\chi \distBound}{\lambda}$.
\end{proof}

\begin{remark}
In the case of continuous-time optimizer \eqref{eq:learning_dynamics_contraction_S}, we have that $\boldsymbol{\theta}_S$ converges exponentially to the equilibrium point $\boldsymbol{\theta}_S^*$ enclosed by the contraction region:

\[||\boldsymbol{\theta}^* - \boldsymbol{\theta}_{S}|| \leq \chi e^{-\lambda t}b\]
Since $\boldsymbol{\theta}_S^*$ is a particular trajectory of the optimizer dynamics, by robustness we also have:
\[||\boldsymbol{\theta}_S^* - \boldsymbol{\theta}_{S'}|| \leq \chi e^{-\lambda t}b + \frac{2\chi \xi }{\lambda n} \]
by the triangle inequality:
\[||\boldsymbol{\theta}_{S} - \boldsymbol{\theta}_{S'}|| \leq 2\chi e^{-\lambda t}b + \frac{2\chi \xi }{\lambda n} \]
this puts the following lower bound on the size of the contraction region $B$:
\[ B > 2b\chi + \frac{2\chi \xi }{\lambda n} \]
If this lower bound is satisfied, then an optimizer initialized within a distance $b$ of a locally exponentially stable equilibrium which is surrounded by a basin of contraction of radius $B > b$ (Figure \ref{fig:multi_loss}) will be algorithmically stable.
\end{remark}

\subsection{Semi-Contracting Optimizers}
\label{subsection:semi_contraction}
If the optimizer is not strictly contracting ($\lambda > 0$), but instead is \textit{semi-contracting} (i.e., $\lambda \geq 0$) then our algorithmic stability bound $\epsilon_{stab}$ is not independent of the training time. This is because the geodesic distance $d_\mathcal{M}(\mathbf{x},\mathbf{x}_p)$ between unperturbed and perturbed trajectories evolves according to:

\[\ddt d_\mathcal{M}(\mathbf{x},\mathbf{x}_p) + \lambda d_\mathcal{M}(\mathbf{x},\mathbf{x}_p) \leq ||\mathbf{T}(\mathbf{x},t)\mathbf{d}(\mathbf{x}_p,t)|| \]
If the only information we have about $\lambda$ is that it is non-negative, then we can only bound the distance between trajectories as:

\[d_\mathcal{M}(\mathbf{x},\mathbf{x}_p) \leq \sup( ||\mathbf{T}(\mathbf{x},t)\mathbf{d}(\mathbf{x}_p,t)||) T + R(0) \]
Considering the disturbance bound $||\mathbf{d}(\mathbf{x}_p,t)|| \le \frac{2\xi}{n}$ leads to the algorithmic stability bound:

\begin{equation}\label{eq:semi-contracting-robust}
 \epsilon \leq L \left[\frac{2\chi \xi}{n}T + R(0) \right]    
\end{equation}
This bound holds generally for semi-contracting systems. However if we know additional information about the dynamics--for example that they are modulated by a decaying learning rate--much tighter bounds can be obtained. We show this with the following example.

\subsubsection{Example: Gradient Flows on Convex Losses}\label{subsubsection:convexlossesGD}
Here we show how the above analysis reproduces a well-known result from \cite{hardt2016train} regarding the algorithmic stability of SGD on convex (but not strongly convex) losses. We suppose that the Hessian of the loss function is positive semi-definite:

\[\nabla^2 \mathcal{L} \geq 0 \]
Consider the gradient flow with learning rate scheduler \citep{goodfellow2016deep}:

\[\thetaDot = -\alpha(t)\nabla \mathcal{L} \]
where $\alpha(t) \geq 0$. This optimizer is semi-contracting in the identity metric, since:
\[\frac{\partial \thetaDot}{\partial \boldsymbol{\theta}}  = -\alpha(t)\nabla^2 \mathcal{L} \leq 0 \]
In this case the disturbance term in Theorem \ref{theorem:Contraction_Implies_AS} is the same as before, just with an addition $\alpha(t)$ factored in. To facilitate comparison with \cite{hardt2016train}, we assume as they do that the optimizer is always initialized at the origin (i.e., $R(0) = 0$). The Euclidean distance between the optimizer trajectories on training sets $S$ and $S'$ evolves according to:

\[\dot{R} \leq \alpha(t)\distBound \]
where $\distBound = 2L/n$. Integrating this inequality and setting $R(0)$ yields the algorithmic stability bound:
\[\epsilon \leq \frac{2 L^2}{n}\int_{t = 0}^T\alpha(t)dt\]
which is the result of \cite{hardt2016train}, Theorem 3.8. We remind the reader that the extra factor of $L$ is picked up from \eqref{eq:lip_ass}. This result helps explain why a decaying learning rate is a useful strategy in deep learning--if the learning rate decays quickly enough (e.g., exponentially), the above integral converges, so that $n$ and $T$ do not `compete' with each other, as they do in \eqref{eq:semi-contracting-robust}.

\begin{remark}
The equivalence between semi-contraction of natural gradient flows and geodesic convexity was recently proven in \cite{wensing2020beyond}. Thus the above algorithmic stability bound extends immediately to this case.
\end{remark}

\subsection{Partial Contraction}
In many cases of interest, a `pure' contraction analysis is hard or difficult to do. For example when an optimizer has an adaptive learning rate, this can significantly complicate the calculation of the Jacobian. To deal with these difficulties, we make use of a generalization of contraction introduced in \cite{wang2005partial}, known as \textit{partial contraction}. 

\begin{definition}[Partially Contracting Dynamical System]
Consider the system \eqref{eq:dynamical_system} (not nessesarily contracting) and an auxiliary system of the form:

\[\dot{\mathbf{y}} = \mathbf{g}(\mathbf{y}, \mathbf{x}, t)  \]

We assume that this auxiliary system for $\y$ is contracting in metric $\mathbf{M}$ with rate $\lambda$. We also assume that $\mathbf{g}(\mathbf{x}, \mathbf{x}, t) = \mathbf{f}(\mathbf{x}, t)$.
If a single, particular trajectory $\mathbf{y}(t)$ of the auxiliary system is known, then all trajectories of \eqref{eq:dynamical_system} converge exponentially towards $\mathbf{y}(t)$. In this case we say that \eqref{eq:dynamical_system} is partially contracting.
\end{definition}

Partially contracting systems are also robust to disturbances \citep{del2012contraction}, a property we will make use of. In particular we have the following theorem:

\begin{theorem}\label{theorem:partial_contraction_robustness}[Robustness of Partially Contracting Systems]
Assume that \eqref{eq:dynamical_system} is partially contracting, and now perturb it with some disturbance $||\mathbf{d}(\mathbf{x},t)|| \leq \distBound$:

\[ \dot{\mathbf{x}}_p = \mathbf{f}(\mathbf{x}_p,t) + \mathbf{d}(\mathbf{x}_p,t)\]
Then after exponential transients of rate $\lambda$, we have the following robustness result:

\[||\mathbf{x} - \mathbf{x}_p|| \leq \frac{\distBound \chi}{\lambda}  \]

\end{theorem}

\subsubsection{Adaptive Learning Rates}
The above results can be extended to include the presence of a state-dependent,time-varying, learning rate \citep{zeiler2012adadelta,kingma2014adam,goodfellow2016deep,liu2019variance}.  In particular consider the learning dynamics with a learning rate scheduler $\rho(\btheta,t)$:

\[ \thetaDot_S = -\rho(\boldsymbol{\theta}_S,t)\mathbf{G}(\boldsymbol{\theta}_S,S) \]
where $\rho(\boldsymbol{\theta},t) \geq \rho_{min} > 0 $. Now consider the auxiliary \textit{virtual system}:

\[ \thetaDot_y = -\rho(\boldsymbol{\theta}_S,t)\mathbf{G}(\boldsymbol{\theta}_y,S) \]
By Theorem \ref{theorem:partial_contraction_robustness}, if the $\btheta_y$ system is contracting in $\mathbf{M}$ with rate $\rho_{min} \lambda$, then we have, by the same arguments in Theorem \ref{theorem:Contraction_Implies_AS}, that $\mathcal{A}(S)$ is asymptotically (after exponential transients of rate $\rho_{min} \lambda)$ uniformly $\epsilon$-stable with:

\[\epsilon_{stab} = \frac{2\chi L \xi }{\rho_{min}\lambda n} \]

\section{Concluding Remarks and Comments}

\subsection{Comparison to Related Work}
Our results are similar in spirit to \cite{charles2018stability}, in the sense that we also use an optimizer's intrinsic dynamical stability to provide generalization error bounds. In certain cases--for example gradient flow on strongly convex losses--our results allow us to derive tighter bounds, because we do not assume the existence of a global minimizer $\boldsymbol{\theta}^*$ and go through the triangle inequality to bound the distance between $\boldsymbol{\theta}_S$ and $\boldsymbol{\theta}_{S'}$.

\subsection{Future Directions}
Contracting systems are robust to noise \citep{pham2013stochastic}, and therefore it seems likely that the results presented here can straightfowardly be extended to stochastic gradient flows--along the lines of \cite{mandt2015continuous} or \cite{boffi2020continuous}. Furthermore, we only examined generalization error here and did not analyze the bias-variance trade-off. Our work also suggests a potential connection to the double descent phenomenon \citep{nakkiran2021deep}. In particular \eqref{eq:robustness_gen_result} implies that the generalization error can overshoot by a factor of $L\chi$, which gives room for the generalization to increase transiently from its initial value before it eventually decreases. This will be explored in future work.

We conclude with some speculations on how the above results relate to biology, specifically neuroscience. The role that non-Euclidean geometry plays in objective-based functions of the brain is an interesting and open question \citep{surace2020choice}. Many local synaptic rules can be thought of as implementing optimization over a loss function. For example in certain settings Hebbian plasticity minimizes Principal Component Loss \citep{oja1992principal}. It seems plausible that our results may be used to quantify the generalization behavior of such rules. We also did not explore the combination properties of contracting systems. Contracting systems can be combined in various forms of hierarchy and feedback in ways which automatically preserve contraction \citep{lohmiller1998contraction,SLOTINE2001137,kozachkov2021recursive}. The results here suggest that combinations of contracting optimizers automatically generalize well--which is a property one could easily imagine evolution would like to preserve in a system like the brain.

\section{Acknowledgments}
This work benefited from stimulating discussions with John Tauber (Neuroscience Statistics Lab at MIT), Akshay Rangaman, Andrzej Banburski-Fahey (Center for Brains, Minds and Machines at MIT), as well as members of the Fiete Lab at MIT. 

%%%%%%%%%%%%%%%%%%%%%%%%%%%%%%%%%%%%%%%%%%%%%%%%%%%%%%%%%%%%

\appendix

\bibliographystyle{plainnat}
\bibliography{stab_imp_stab.bib}

\end{document}